\documentclass[letterpaper]{article} % DO NOT CHANGE THIS
\usepackage{aaai25}  % DO NOT CHANGE THIS
\usepackage{times}  % DO NOT CHANGE THIS
\usepackage{helvet}  % DO NOT CHANGE THIS
\usepackage{courier}  % DO NOT CHANGE THIS
\usepackage[hyphens]{url}  % DO NOT CHANGE THIS
\usepackage{graphicx} % DO NOT CHANGE THIS
\usepackage{natbib}  % DO NOT CHANGE THIS AND DO NOT ADD ANY OPTIONS TO IT
\usepackage{caption} % DO NOT CHANGE THIS AND DO NOT ADD ANY OPTIONS TO IT
\usepackage{algorithm}
\usepackage{amsmath}
\usepackage{amsfonts}
\usepackage{amsthm}
\newtheorem{theorem}{Theorem}
\usepackage{amssymb}
\usepackage{algpseudocode}

\usepackage{pifont}
\usepackage{multirow}
\usepackage{booktabs}
\usepackage{mathtools}
\usepackage{subcaption}
\usepackage {color}
\usepackage{xcolor}
\usepackage{newfloat}
\usepackage{listings}
\DeclareCaptionStyle{ruled}{labelfont=normalfont,labelsep=colon,strut=off} % DO NOT CHANGE THIS
\floatstyle{ruled}
\newfloat{listing}{tb}{lst}{}
\floatname{listing}{Listing}

\newcommand{\eat}[1]{}
\ifodd 2

\newcommand{\han}[1]{{\color{purple}{#1}}}
\newcommand{\TODO}[1]{{\color{red}TODO:{#1}}}
\newcommand\beftext[1]{{\color[rgb]{0.5,0.5,0.5}{BEFORE:#1}}}
\else

\newcommand{\TODO}[1]{}
\newcommand{\beftext}[1]{}

\newcommand{\han}[1]{#1}
\fi
\usepackage{hyperref}

\title{Erase then Rectify: A Training-Free Parameter Editing Approach for Cost-Effective Graph Unlearning}
\author{
    Zhe-Rui Yang\textsuperscript{\rm 1,2,3}, Jindong Han\textsuperscript{\rm 4}, Chang-Dong Wang\textsuperscript{\rm 1,3,*}, Hao Liu\textsuperscript{\rm 2,4,*}
}
\affiliations{
    \textsuperscript{\rm 1}Sun Yat-sen University, Guangzhou, China\\
    \textsuperscript{\rm 2}The Hong Kong University of Science and Technology (Guangzhou), Guangzhou, China\\
    \textsuperscript{\rm 3}Guangdong Key Laboratory of Big Data Analysis and Processing, Guangzhou, China\\
    \textsuperscript{\rm 4}The Hong Kong University of Science and Technology, Hong Kong, China\\
    yangzhr9@mail2.sysu.edu.cn, jhanao@connect.ust.hk, wangchd3@mail.sysu.edu.cn, liuh@ust.hk\\
    \textsuperscript{\rm *}Corresponding author\\
}

\usepackage{bibentry}
\begin{document}

\maketitle

\begin{abstract}
Graph unlearning, which aims to eliminate the influence of specific nodes, edges, or attributes from a trained Graph Neural Network (GNN), is essential in applications where privacy, bias, or data obsolescence is a concern. However, existing graph unlearning techniques often necessitate additional training on the remaining data, leading to significant computational costs, particularly with large-scale graphs. To address these challenges, we propose a two-stage training-free approach, \emph{Erase then Rectify}~(ETR), designed for efficient and scalable graph unlearning while preserving the model utility. Specifically, we first build a theoretical foundation showing that masking parameters critical for unlearned samples enables effective unlearning. Building on this insight, the Erase stage strategically edits model parameters to eliminate the impact of unlearned samples and their propagated influence on intercorrelated nodes. To further ensure the GNN's utility, the Rectify stage devises a gradient approximation method to estimate the model's gradient on the remaining dataset, which is then used to enhance model performance. Overall, ETR achieves graph unlearning without additional training or full training data access, significantly reducing computational overhead and preserving data privacy. Extensive experiments on seven public datasets demonstrate the consistent superiority of ETR in model utility, unlearning efficiency, and unlearning effectiveness, establishing it as a promising solution for real-world graph unlearning challenges.

\begin{links}
  \link{Code}{https://github.com/AllminerLab/ETR}
\end{links}
\end{abstract}

% Uncomment the following to link to your code, datasets, an extended version or similar.
%
% \begin{links}
%     \link{Code}{https://aaai.org/example/code}
%     \link{Datasets}{https://aaai.org/example/datasets}
%     \link{Extended version}{https://aaai.org/example/extended-version}
% \end{links}

\section{Introduction}
Machine unlearning has garnered attention for its ability to remove the influence of specific data subsets from a well-trained machine learning model~\cite{DBLP:journals/corr/abs-2209-02299}, which is crucial for forgetting sensitive, mislabeled, or outdated information~\cite{DBLP:journals/jmlr/LiuT17}. 
Typically, the subset of data to be unlearned is much smaller than the entire training dataset. As a result, one primary goal of machine unlearning is to efficiently eliminate the impact of these unlearned samples while preserving the model’s predictive power, thereby avoiding the costly process of retraining the model from scratch on the remaining data~\cite{DBLP:journals/csur/XuZZZY24}.

In the graph domain, there is also a strong demand to forget specific information, such as users requesting the removal of personal data in social networks. 
In particular, Graph Neural Networks (GNNs) are widely used to process graph-structured data through message propagation and neighborhood aggregation~\cite{DBLP:journals/aiopen/ZhouCHZYLWLS20}. However, due to the interconnected nature of graphs, unlearned samples can significantly affect their neighbors, making conventional machine unlearning methods unapplicable~\cite{DBLP:conf/iclr/ChengDHAZ23}. As a result, graph unlearning has emerged as a specialized task aimed at removing unlearned samples and mitigating their impact on other nodes~\cite{DBLP:journals/corr/abs-2310-02164}.

Recently, several graph unlearning approaches have been proposed. For instance, GIF~\cite{DBLP:conf/www/WuYQS0023} and CGU~\cite{DBLP:conf/iclr/Chien0M23} approximate parameter changes caused by data removal and update the affected parameters. However, these methods primarily focus on forgetting unlearned samples while overlooking the predictive capability on the remaining data, which can significantly compromise model utility~\cite{DBLP:conf/aaai/LiZWZLW24}. To address this, some recent studies~\cite{DBLP:conf/aaai/LiZWZLW24, DBLP:conf/uss/WangH023} consider both predictive and unlearning objectives when an unlearning request is received. However, they typically require additional training on the remaining data, leading to high computational overhead, particularly with large graphs, even when only a few samples are unlearned. This substantial resource demand limits the efficiency and scalability of these methods in real-world applications.

In this work, we focus on developing an efficient and scalable graph unlearning method that preserves model utility. However, achieving this goal is non-trivial. First, the interconnected nature of graph data means that unlearned samples can significantly affect their neighbors through message propagation, complicating the removal of both the unlearned samples and their influence. Second, the remaining dataset is typically much larger than the unlearned subset, leading to substantial computational overhead when accessing the remaining data. Balancing model utility with low computational cost presents another significant challenge.

To address these challenges, we propose \textbf{E}rase \textbf{t}hen \textbf{R}ectify (ETR), a training-free, two-stage approach for cost-effective graph unlearning. First, we establish a theoretical foundation showing that masking parameters critical to unlearned samples enables effective unlearning. Based on this, in the Erase stage, we propose a neighborhood-aware parameter editing strategy to remove the impact of unlearned samples and their cascading effects, while minimizing the impact on model utility. In the Rectify stage, we introduce a subgraph-based gradient approximation method to estimate the unlearned model’s gradient on the remaining data, further enhancing model performance. Notably, ETR achieves graph unlearning without requiring explicit model training or access to the entire training dataset, avoiding computational overhead and ensuring data privacy. Overall, the main contributions are as follows:

\begin{itemize}
  \item We theoretically prove that masking parameters crucial for unlearned samples enable effective graph unlearning.
  \item We propose a training-free neighborhood-aware parameter editing method for graph unlearning that effectively forgets unlearned samples and their influence on intercorrelated samples.
  \item We propose a gradient approximation method to enhance GNN performance on the remaining graph without requiring access to the entire training dataset, thereby ensuring efficiency and scalability on large-scale graphs.
  \item Extensive experiments demonstrate the model utility, unlearning efficiency, and unlearning efficacy of ETR. Notably, ETR achieves an average of 4583.9x less time and 4.2x less memory usage than retraining from scratch.
\end{itemize}

\section{Related work}
\subsection{Graph Unlearning}
Graph unlearning methods are divided into exact and approximate approaches, aiming to obtain a model that is exactly or approximately equivalent to retraining from scratch~\cite{DBLP:journals/corr/abs-2310-02164}. Current efforts primarily focus on approximate methods. For example, GraphEraser~\cite{DBLP:conf/ccs/Chen000H022} and GUIDE~\cite{DBLP:conf/uss/WangH023} follow the SISA~\cite{DBLP:conf/sp/BourtouleCCJTZL21} paradigm, partitioning the training set into multiple shards, with a separate model trained for each shard. Other approaches~\cite{DBLP:conf/www/WuYQS0023, DBLP:conf/kdd/WuS0WW23, DBLP:conf/iclr/ChenLLH23, DBLP:conf/iclr/Chien0M23} use influence functions to approximate the impact of unlearning on model parameters. GNNDelete~\cite{DBLP:conf/iclr/ChengDHAZ23} freezes model parameters while training additional parameters for unlearning. MEGU~\cite{DBLP:conf/aaai/LiZWZLW24} simultaneously trains the model’s prediction and unlearning objectives. Despite their effectiveness, these methods still face scalability issues with large graphs.

\subsection{Training-Free Machine Unlearning}
Several training-free methods have been proposed for efficient machine unlearning. Fisher Forgetting~\cite{DBLP:conf/cvpr/GolatkarAS20} and NTK~\cite{DBLP:conf/eccv/GolatkarAS20} are weight scrubbing methods that add noise to parameters informative for unlearned samples. However, Fisher Forgetting is computationally expensive and harms predictive performance~\cite{10113700}, while NTK requires additional models for unlearning. The closest work to ours is SSD~\cite{DBLP:conf/aaai/FosterSB24}, which dampens parameters crucial for the unlearned dataset while preserving others. Unlike SSD, we account for the impact of unlearned samples on other samples through message propagation. We also propose adaptive hyperparameter selection to improve versatility across datasets and introduce the Rectify strategy to enhance performance on remaining data.

\section{Preliminaries}
\subsection{Notations and Background}
In this paper, we employ $\mathcal{G}=\left( \mathcal{V}, \mathcal{E}, \mathcal{X} \right)$ to denote a graph comprising $\left| \mathcal{V} \right|$ nodes and $\left| \mathcal{E} \right|$ edges. Each node $\mathbf{v}_i \in \mathcal{V}$ has a $d$-dimensional feature vector $\mathbf{x}_i \in \mathcal{X}$. A major category of GNNs is message-passing neural networks (MPNNs)~\cite{DBLP:conf/icml/GilmerSRVD17}, such as GCN~\cite{DBLP:conf/iclr/KipfW17} and GAT~\cite{DBLP:conf/iclr/VelickovicCCRLB18}. MPNNs propagate and aggregate features from neighboring nodes, then transform these aggregated features to update node representations~\cite{DBLP:conf/icml/GilmerSRVD17}. Following previous works~\cite{DBLP:conf/ccs/Chen000H022, DBLP:journals/corr/abs-2310-02164}, we focus on node classification in this study. 
We denote the training dataset as $D$, the unlearned dataset as $D_f$, the remaining dataset as $D_r$ (i.e., $D \setminus D_f$), and the $k$-hop neighborhood of $D_f$ as $D_k$.
\begin{figure*}[!t]
    \centering
    \includegraphics[width=0.9\linewidth]{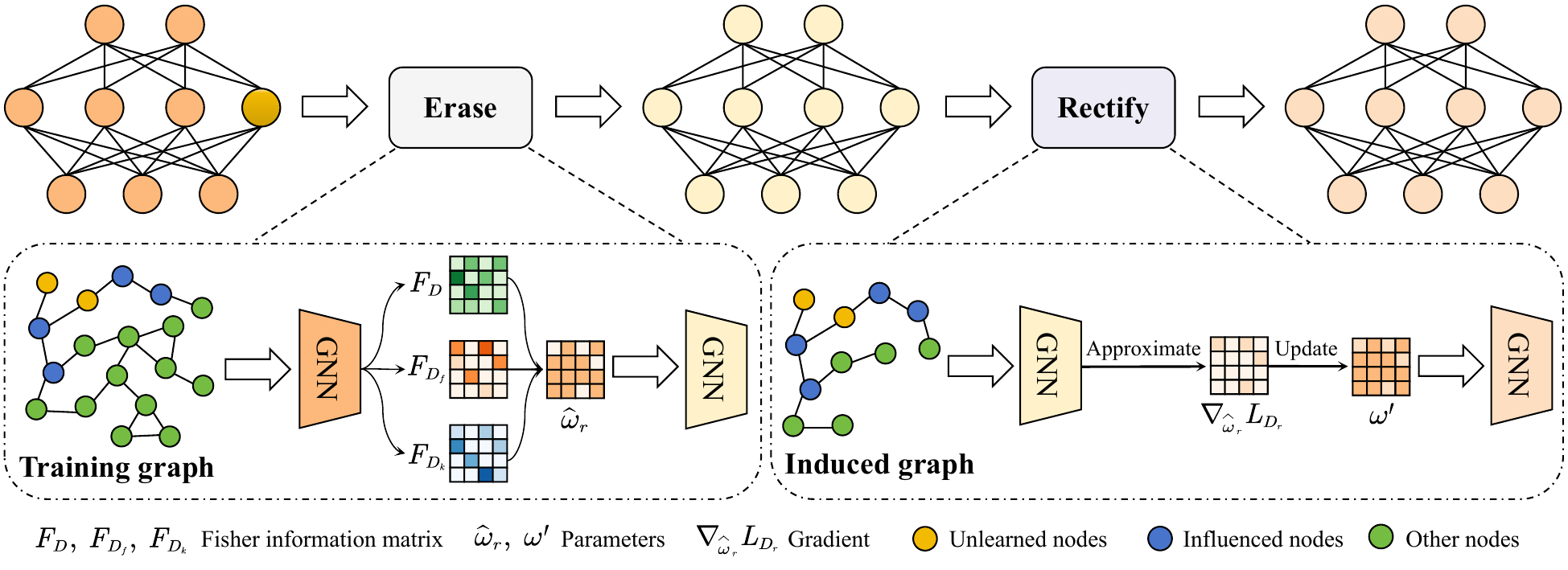}
    \caption{The framework of ETR: Forgetting the influence of unlearned samples through Erase, followed by enhancing the model performance on the remaining dataset via Rectify.}
    % \vspace{-10pt}
    \label{fig:framework}
\end{figure*}

\subsection{Fisher Information Matrix (FIM)}
The FIM uses the second-order derivative of the loss function to quantify parameter sensitivity, providing a measure of importance with respect to input samples~\cite{DBLP:conf/icml/GuoGHM20}. For a distribution $p(y|x, w)$, the FIM and its first-order derivative property~\cite{kay1993fundamentals} are given as follows:

\begin{small}
\begin{align}
    F_D =& \mathbb{E}_{x,y}\left[ -\nabla _{w}^{2}\log p\left( y|x,w \right) \right] \\
    =& \mathbb{E}_{x,y}\left[ \nabla _w\log p\left( y|x,w \right) \nabla _w\log p\left( y|x,w \right) ^T \right] 
    % \\
    % \approx& \frac{1}{\left| D \right|}\sum_{i=1}^{\left| D \right|}{\nabla _w\log p\left( y_i|x_i,w \right) \nabla _w\log p\left( y_i|x_i,w \right) ^T}
    \label{FIM}
\end{align}
\end{small}

Notably, obtaining the FIM is computationally expensive. A common approach is to use its diagonal $diag(F)$~\cite{doi:10.1073/pnas.1611835114}, which can be computed efficiently using the first-order derivative. The $i$-th diagonal element of the FIM over dataset $D$ is denoted as $F_{D, ii}$, representing the importance of the $i$-th parameter with respect to $D$. \eat{\han{the details of diagonal approximation is unclear to me.}}

\subsection{Problem Statement}
Given a graph $\mathcal{G}$, the optimal GNN model $f_\mathcal{G}$ trained on $\mathcal{G}$, and a graph unlearning request $\Delta \mathcal{G}=\left( \Delta \mathcal{V}, \Delta \mathcal{E}, \Delta \mathcal{X} \right)$, the objective of graph unlearning is to derive a new GNN model $\widehat{f}$ that minimizes the discrepancy between $\widehat{f}$ and $f_{\mathcal{G} \setminus \Delta \mathcal{G}}$. Here, $f_{\mathcal{G} \setminus \Delta \mathcal{G}}$ represents the GNN model retrained from scratch on the remaining graph $\mathcal{G} \setminus \Delta \mathcal{G}$. 

Graph unlearning can be categorized into three tasks:
\textbf{Node Unlearning}: $\Delta \mathcal{G}=\left( \Delta \mathcal{V}, \Delta \mathcal{E}, \Delta \mathcal{X} \right)$, where $\Delta \mathcal{V}$ represents unlearned nodes, $\Delta \mathcal{E}$ denotes edges connecting to $\Delta \mathcal{V}$, and $\Delta \mathcal{X}$ represents the features of nodes in $\Delta \mathcal{V}$. \textbf{Edge Unlearning}: $\Delta \mathcal{G}=\left( \emptyset, \Delta \mathcal{E}, \emptyset \right)$, where $\Delta \mathcal{E}$ represents the set of edges to be unlearned. \textbf{Feature Unlearning}: $\Delta \mathcal{G}=\left( \emptyset, \emptyset, \Delta \mathcal{X} \right)$, where $\Delta \mathcal{X}$ represents the unlearned features, and will be padded with zeros. We will use the node unlearning task to illustrate the proposed method, and describe how the proposed method applies to edge unlearning and feature unlearning tasks in the Appendix.

\section{Methodology}
\subsection{Framework Overview}
The overall framework of ETR is depicted in~\figurename~\ref{fig:framework}. In the Erase stage, we evaluate parameter importance with respect to input samples based on the FIM. Subsequently, we identify \eat{parameters remembering}\han{critical parameters concerning} the unlearned samples and their influence on other samples. 
We then modify these parameters to erase the \eat{influence}\han{effect} of the unlearned samples.
In the Rectify stage, we utilize the induced graph of unlearned samples to approximate model gradients on the remaining dataset, which are then used to enhance the model performance on the remaining dataset. \eat{\han{use what to rectify? based on the approximated gradients?}}

\subsection{Graph Unlearning via Parameter Masking}
Previous research indicates that certain neurons or parameters are critical in memorizing specific samples~\cite{DBLP:conf/icml/MainiMSLKZ23}, and that gradient manipulation can prevent the memorization of noisy information~\cite{DBLP:conf/nips/ChenZCDLLD21}.
Motivated by these findings, we propose to achieve efficient unlearning by directly masking the parameters responsible for memorizing the data to be unlearned, as described below:

\begin{small}
    \begin{equation}
    \label{mask}
    \widehat{\omega}_{r,j}=\begin{cases}
	   0,&		j\in M\\
	   \omega _{j}^{*},&		j\notin M\\
    \end{cases}
    \end{equation}
\end{small}
where $\omega ^{*}$ represents the optimal parameters on $D$, and $M$ denotes the set of parameters responsible for memorizing the unlearned dataset. To explore the effectiveness of the masking strategy, we provide theoretical justification for its ability to achieve graph unlearning as follows:

\begin{theorem}
\label{theorem}
For a GNN model, if we approximate the FIM $F$ with its diagonal $diag(F)$, and assume all elements of $diag(F)$ are strictly positive, the mean squared distance between the optimal model parameters trained on $D_r$ and the parameters obtained from (\ref{mask}), denoted as $Q = \frac{1}{\left| \omega \right|}\sum_j{\left( \omega _{r,j}^{*}-\,\,\widehat{\omega}_{r,j} \right)}^2$, has the following upper bound:
\begin{small}
    \begin{equation}
    Q \leqslant \frac{1}{\left| \omega \right|}\left( c_1\sum_{j\in M}{\frac{1}{F_{D_r,jj}^{2}}}+c_2\sum_{j\notin M}{\frac{F_{D_f,jj}^{2}}{F_{D,jj}^{2}F_{D_r,jj}^{2}}}+c_3 \right) 
    \end{equation}
\end{small}
where $c_1$, $c_2$, $c_3$ are constants, and $|\omega|$ denotes the number of parameters. \eat{\han{$|\omega|$ is not formally defined}}
\end{theorem}

The proof of Theorem~\ref{theorem} is provided in the Appendix. According to Theorem~\ref{theorem}, adding the $j$-th parameter to $M$ will increase the upper bound of $Q$ by $\frac{1}{\left| \omega \right|}\left( \frac{c_1}{F_{D_r,jj}^{2}}-\frac{c_2F_{D_f,jj}^{2}}{F_{D,jj}^{2}F_{D_r,jj}^{2}} \right) $. This increase is less than 0 when $\frac{F_{D_f,jj}}{F_{D,jj}}>\sqrt{\frac{c_1}{c_2}} = c$, indicating that the parameter $w_j$ is much more important for $D_f$ than for $D$. Therefore, masking parameters that are crucial for unlearned samples but not for others can reduce the upper bound of $Q$, facilitating the forgetting of the unlearned dataset.

\subsection{Erase: Neighborhood-Aware Parameter Editing}
Despite the above analysis theoretically guaranteeing the effectiveness of the masking strategy in achieving unlearning, it has two major limitations
First, the masking strategy overlook the extent to which $\frac{F_{D_f,jj}}{F_{D,jj}}$ exceeds $c$. In particular, parameters for which $\frac{F_{D_f,jj}}{F_{D,jj}}$ significantly exceeds $c$ are much more critical for the unlearned dataset compared to those where $\frac{F_{D_f,jj}}{F_{D,jj}}$ only slightly exceeds $c$. Second, the masking strategy overlooks the impact of unlearned samples on their neighbors through message propagation.

To address these, we propose the Erase strategy as:
\begin{small}
    \begin{equation}
    \widehat{\omega}_{r,j}=\begin{cases}
	   a \frac{F_{D,jj}}{F_{D_f,jj}}\omega _{j}^{*},& F_{D_f,jj} > \gamma F_{D,jj},\\
    b \frac{{F_{D,jj}}^2}{F_{D_f,jj}F_{D_{k},jj}}\omega _{j}^{*},& F_{D_f,jj}F_{D_k,jj}>{\eta F_{D,jj}}^2\\ &and \ F_{D_f,jj} \leq \gamma F_{D,jj},\\
	\omega _{j}^{*},&		otherwise.\\
\end{cases}
\end{equation}
\end{small}
where $a$, $b$, $\gamma$, and $\eta$ are hyperparameters. 

The advantages of the above design are threefold. 
First, it considers the impact of unlearned samples on their neighbors.
According to the properties of the FIM, 
a large ratio of $\frac{F_{D_f,jj}}{F_{D,jj}}$ or $\frac{F_{D_k,jj}}{F_{D,jj}}$ indicates that the parameter is crucial for $D_f$ or $D_k$, but not for $D$.
Drawing inspiration from Theorem 1, if a parameter is crucial for both $D_f$ and $D_k$, but not for other samples, i.e., satisfying $F_{D_f,jj}F_{D_k,jj} > \eta {F_{D,jj}}^2$, we consider that the parameter has been influenced by message propagation from $D_f$. In such cases, we modify these parameters to forget the influence of unlearned samples.

Second, when modifying parameters, we account for the extent to which they are influenced by unlearned samples. 
Specifically, rather than directly masking the parameters, we use the coefficients $\frac{F_{D,jj}}{F_{D_f,jj}}$ and $\frac{{F_{D,jj}}^2}{F_{D_f,jj}F_{D_{k},jj}}$ to determine the degree of modification. This implies that a large $\frac{F_{D_f,jj}}{F_{D,jj}}$ or $\frac{{F_{D_f,jj}F_{D_{k},jj}}}{{F_{D,jj}}^2}$ necessitates substantial modifications to $\omega _{j}^{*}$ \eat{\han{do you mean $\omega _{j}^{*}$?}}, while a small value requires only minor modifications.

Finally, we balance the two forgetting objectives: forgetting $D_f$ and forgetting the influence of $D_f$ on $D_k$. One of these objectives may dominate the other due to the potential difference in magnitude between $\frac{F_{D,jj}}{F_{D_f,jj}}$ and $\frac{{F_{D,jj}}^2}{F_{D_f,jj}F_{D_{k},jj}}$. To mitigate this, we introduce two balancing coefficients $a$ and $b$ to balance these two forgetting objectives. \eat{\han{$a$, $b$ or $a$, $b$?}}

Note the choice of hyperparameters may vary depending on the dataset. To facilitate hyperparameter selection, we adaptively select them for different datasets. Specifically, $\gamma$ is set as the top $m$\% \eat{\han{can you use another notation to indicate $k$?}} value of $\frac{F_{D_f}}{F_{D}}$, and $\eta$ is set as the top $m$\% value of $\frac{F_{D_f}F_{D_k}}{{F_{D}}^2}$. This ensures that $m$\% of the parameters satisfy $F_{D_f,jj} > \gamma F_{D,jj}$ and $m$\% satisfy $F_{D_f,jj}F_{D_k,jj} > {\eta F_{D,jj}}^2$. Furthermore, we set $a = \gamma$ and $b = \eta$ to balance the magnitudes of $\frac{F_{D,jj}}{F_{D_f,jj}}$ and $\frac{{F_{D_jj}}^2}{F_{D_f,jj}F_{D_{k},jj}}$.

\begin{algorithm}[t]
\caption{ETR: Erase then Rectify}
\label{alg}
\begin{algorithmic}[1]
\Require  
    Unlearned dataset $D_f$; Dataset $D_k$; GNN $f_\mathcal{G}$; Optimal parameters $\omega^{*}$; Gradient $\nabla _{\omega^{*}}L_D$; Hyperparameters $m$ and $\lambda$.
\Statex \textit{\textbf{Erase}}
    \State Calculate the gradients $\nabla _{\omega^{*}}L_{D_f}$ and $\nabla _{\omega^{*}}L_{D_k}$.
    \State Calculate the FIM $F_{D}$, $F_{D_f}$, $F_{D_k}$ via $(\ref{FIM})$.
    \State Obtain $\gamma=top\mbox{-}m\%(\frac{F_{D_f}}{F_{D}})$ and $\eta=top\mbox{-}m\%(\frac{F_{D_f}F_{D_{k}}}{{F_{D}}^2})$
    \For{$j \ in\  range \ |\omega|$}
        \State Obtain $\alpha _j=\frac{F_{D,jj}}{F_{D_f,jj}}$ and $\beta _i=\frac{{F_{D,jj}}^2}{F_{D_f,jj}F_{D_{k},jj}}$
        \If{$F_{D_f,jj} > \gamma F_{D,jj}$}
            \State $\widehat{\omega}_{r,j}=\alpha _j\gamma \omega _j^*$
        \ElsIf{$F_{D_f,jj}F_{D_k,jj}>{\eta F_{D,jj}}^2$}
            \State $\widehat{\omega}_{r,j}=\beta _j\eta \omega _j^*$
        \Else
            \State $\widehat{\omega}_{r,j}=\omega _j^*$
        \EndIf
    \EndFor
    
\Statex \textit{\textbf{Rectify}}
    \State Calculate the gradient $\nabla _{\widehat{\omega}_{r}}L_{D_k}$.
    \State Obtain the gradient $\nabla _{\widehat{\omega}_{r}}L_{D_r}$ via $(\ref{gradient})$.
    \State Obtain the parameters $\omega ^{\prime}$ via $(\ref{update})$.
\Ensure 
    $\omega ^{\prime}$.
\end{algorithmic}
\end{algorithm}

\begin{table*}[!t]
\small
\setlength{\tabcolsep}{2mm}
\centering
% \resizebox{\textwidth}{!}{
\begin{tabular}{cc|cccccccccc}
\hline
\multirow{2}{*}{Bone} & \multirow{2}{*}{Method} & \multicolumn{2}{c}{PubMed}            & \multicolumn{2}{c}{CiteSeer}          & \multicolumn{2}{c}{Cora}              & \multicolumn{2}{c}{CS}                & \multicolumn{2}{c}{Physics}        \\ 
% \cline{3-14} 
%                           &                         & F1                  & T               & F1                  & T               & F1                  & T               & F1                  & T               & F1                  & T               & F1                  & T               \\ \hline
\cmidrule(l){3-4} \cmidrule(l){5-6} \cmidrule(l){7-8} \cmidrule(l){9-10} \cmidrule(l){11-12} 
&                         & F1                  & T               & F1                  & T               & F1                  & T               & F1                  & T               & F1                  & T               \\ \hline
\multirow{8}{*}{GCN}      & Retrain                 & 89.74±0.27          & 39.1          & 76.76±0.72          & 14.1          & 84.35±1.14          & 7.5           & 92.63±0.30          & 285.8         & 95.90±0.10          & 856.7         \\
                          & BEKM        & 80.03±0.72          & 68.2          & 61.11±2.27          & 12.7          & 66.42±3.33          & 11.0          & 89.03±0.49          & 71.7          & 94.53±0.19          & 152.7        \\
                          & BLPA       & 81.63±0.83          & 66.1          & 57.69±2.31          & 13.0          & 61.51±3.48          & 11.5          & 88.88±0.56          & 68.2          & 94.45±0.24          & 142.8        \\
                          & GIF                     & 73.66±0.26          & 0.4           & 62.88±1.34          & \underline{0.2}           & 75.87±1.86          & \underline{0.2}           & 85.19±0.56          & 1.3           & 91.37±0.23          & 3.3           \\
                          & SR                & 88.04±0.25          & 29.0          & 75.35±0.70          & 8.6           & 81.25±0.95          & 8.6           & 91.77±0.28          & 49.0          & 94.72±0.14          & 166.4             \\
                          & Fast              & 88.11±0.24          & 29.1          & 75.83±0.88          & 8.8           & 80.70±1.70          & 8.4           & 91.64±0.29          & 47.8          & 94.74±0.16          & 138.3             \\
                          & GNNDelete         & 84.38±1.27          & 2.32s          & 74.77±0.94          & 1.80s           & 78.82±1.50          & 1.66s           & 90.09±0.59          & 3.81s          & 85.84±3.07          & 6.48s             \\
                          & MEGU                    & \underline{88.26±0.47}          & \underline{0.3}           & \underline{76.04±0.93}          & \underline{0.2}           & \underline{82.62±1.05}          & \underline{0.2}           & \underline{92.39±0.51}          & \underline{0.5}           & \underline{95.40±0.20}          & \underline{0.9}           \\
                          & ETR                     & \textbf{89.48±0.34} & \textbf{0.02} & \textbf{77.09±0.60} & \textbf{0.02} & \textbf{83.54±0.92} & \textbf{0.02} & \textbf{92.72±0.37} & \textbf{0.04} & \textbf{95.71±0.15} & \textbf{0.07} \\ \hline
\multirow{8}{*}{GAT}      & Retrain                 & 88.42±0.36          & 50.2          & 76.52±1.00          & 16.3          & 82.29±1.00          & 9.0           & 92.57±0.43          & 295.47         & 96.26±0.18          & 875.2         \\
                          & BEKM        & 70.30±0.77          & 102.2         & 49.64±1.36          & 23.5          & 56.79±4.83          & 20.7          & 81.10±1.39          & 114.3         & 92.04±0.43          & 214.5        \\
                          & BLPA       & 71.65±0.78          & 98.2          & 49.52±2.44          & 20.1          & 49.70±3.17          & 18.7          & 81.35±0.58          & 94.2          & 91.82±0.43          & 194.5        \\
                          & GIF                     & \underline{88.07±0.22}          & 1.1           & 75.80±1.06          & 0.7           & \underline{81.70±1.38}          & 0.6           & \textbf{92.79±0.36} & 2.1           & \underline{95.84±0.16}          & 5.1          \\
                          & SR                & 86.67±0.41          & 33.6          & 75.92±0.88          & 12.0          & 76.05±1.78          & 12.3          & 89.49±0.64          & 54.8          & 94.76±0.41          & 149.3             \\
                          & Fast              & 86.60±0.33          & 35.4          & 75.89±0.93          & 12.4          & 75.46±0.93          & 12.3          & 89.70±0.95          & 55.3          & 94.86±0.45          & 145.5             \\
                          & GNNDelete         & 83.09±3.67          & 3.09s          & 67.21±4.36          & 2.57s           & 76.31±2.65          & 2.51s           & 91.08±1.22          & 4.75s          & 94.43±0.95          & 8.12s             \\
                          & MEGU                    & 85.23±0.47          & \underline{0.3}           & \underline{76.07±0.79}          & \underline{0.2}           & \textbf{82.84±0.72} & \underline{0.2}           & 91.74±0.32          & \underline{0.6}           & 95.14±0.22          & \underline{1.2}                      \\
                          & ETR                     & \textbf{88.32±0.22} & \textbf{0.03} & \textbf{76.31±1.59} & \textbf{0.03} & 79.70±0.85         & \textbf{0.03} & \underline{92.45±0.13}          & \textbf{0.05} & \textbf{96.19±0.18} & \textbf{0.09}   \\ \hline
\end{tabular}% }
\caption{Performance comparison in terms of F1 score and running time, where F1 indicates the F1 score, T indicates the running time (in seconds). The best experimental results are highlighted in bold, while the second-best results are underscored.}
\label{tab:result}
\end{table*}

\subsection{Rectify: Model Utility Enhancement}
Although the Erase strategy mitigates the impact on the model's predictive performance by considering the extent to which unlearned samples influence parameters, the parameter editing approach may still \eat{impair}\han{negatively affect} the GNN's performance on $D_r$ to some \eat{extent}\han{degree}. 
In this part, we propose the Rectify strategy to enhance GNN performance on $D_r$ without requiring access to the entire training dataset, \eat{thereby avoiding substantial computational overhead and ensuring data privacy and security.}\han{which minimizes computational overhead while ensuring data privacy and security.}

Firstly, we \eat{approximate}\han{estimate} the gradients of the edited model on $D_r$ using the induced graph of $D_f$
%, thus avoiding reliance on $D_r$. 
Specifically, we assume that we can store the gradients from the model's final iteration of the training phase. After removing $D_f$, the neighbor structure of $D_r \setminus D_k$ remains unchanged, while only the neighbor structures of $D_f$ and $D_k$ undergo changes. Additionally, during the Erase stage, the \eat{parameters that are modified}\han{modified parameters} are crucial only for $D_f$ and $D_k$, but not for $D_r \setminus D_k$. Therefore, for nodes in $D_r \setminus D_k$, their gradients can be considered nearly unchanged, i.e., 
\begin{small}
    \begin{align}
        \nabla _{\widehat{\omega }_r}l_j\approx \nabla _{\omega ^*}l_j, \forall j\in D_r \setminus D_k
    \end{align}
\end{small}
where $l_j$ denotes the loss of the GNN on the $j$-th sample. Through the aforementioned approximation, we can easily calculate the gradient of the edited model on $D_r$ as follows:
\begin{small}
\begin{align}
\nabla _{\widehat{\omega }_r}L_{D_r}=&\small{\frac{1}{\left| D_r \right|}}\sum_{j\in D_r}{\nabla _{\widehat{\omega }_r}l_j} \\
\approx& \small{\frac{1}{\left| D_r \right|}}\left( \sum_{j\in D_r \setminus D_k}{\nabla _{\omega^{*}}l_j}+\sum_{j\in D_k}{\nabla _{\widehat{\omega }_r}l_j} \right) 
\end{align}
\end{small}
where $L_D$ denotes the loss function of the GNN on dataset $D$. Further, leveraging the gradient of the GNN on $D$, we obtain $\sum_{j\in D_r \setminus D_k}{\nabla _{\omega^{*}}l_j} = \left| D \right|\nabla _{\omega^{*}}L_D-\sum_{j\in D_f}{\nabla _{\omega^{*}}l_j}-\sum_{j\in D_k}{\nabla _{\omega^{*}}l_j}$. Finally, we estimate the gradient of the edited model on $D_r$ as follows:
\begin{small}
\begin{equation}
\begin{aligned}
    \nabla _{\widehat{\omega }_r}L_{D_r} =\small{\frac{1}{\left| D_r \right|}}( &\left| D \right|\nabla _{\omega^{*}}L_D-\left| D_f \right|\nabla _{\omega^{*}}L_{D_f} -\\
    &\left| D_k \right|\nabla _{\omega^{*}}L_{D_k}+\left| D_k \right|\nabla _{\widehat{\omega }_r}L_{D_k} )
    \label{gradient}
\end{aligned}
\end{equation}
\end{small}

Notably, $\nabla _{\omega^{*}}L_D$ can be stored during GNN training, while the sizes of $D_f$ and $D_k$ are typically small. 
Therefore, we can efficiently approximate $\nabla _{\widehat{\omega }_r}L_{D_r}$ through $(\ref{gradient})$. After obtaining the gradient $\nabla _{\widehat{\omega }_r}L_{D_r}$ , we rectify the parameters through \han{one-step} gradient descent, \han{defined} as follows:
\begin{small}
    \begin{equation}
    \omega ^{\prime}=\widehat{\omega }_r -\lambda \nabla _{\widehat{\omega }_r}L_{D_r}
    \label{update}
\end{equation}
\end{small}
where $\lambda$ is the hyperparameter. 
% Through parameter rectification, we enhance the performance of the GNN model on the remaining dataset $D_r$. 
The pseudo-code for the ETR approach is presented in Algorithm~\ref{alg}.

For edge and feature unlearning, we can proceed in a similar manner. The pseudo-code for edge and feature unlearning is provided in the Appendix.

\subsection{Complexity Analysis}
In this section, we analyze the complexity of ETR. Calculating $\nabla_{\omega^{*}}L_{D_f}$ with a time complexity of $\mathcal{O}(|D_f|)$. Similarly, calculating $\nabla_{\omega^{*}}L_{D_k}$ and $\nabla_{\widehat{\omega}_{r}}L_{D_k}$ with a time complexity of $\mathcal{O}(|D_k|)$. 
Computing the FIM has a time complexity of $\mathcal{O}(|\omega|)$.
The Erase strategy can be executed in parallel, with a time complexity of $\mathcal{O}(|\omega|)$. Thus, the overall time complexity of ETR is $\mathcal{O}(|D_f| + |D_k| + |\omega|)$. Regarding space complexity, the GNN parameters have a space complexity of $\mathcal{O}(|\omega|)$. The node features have a space complexity of $\mathcal{O}(|D_f|d + |D_k|d)$. The edges have a space complexity of $\mathcal{O}(|\mathcal{E}_f| + |\mathcal{E}_k|)$, where $|\mathcal{E}_f|$ and $|\mathcal{E}_k|$ denote the number of edges in $D_f$ and $D_k$, respectively. 
Computing the FIM has a space complexity of $\mathcal{O}(|\omega|)$.
Therefore, the overall space complexity of ETR is $\mathcal{O}(|D_f|d + |D_k|d + |\omega| + |\mathcal{E}_f| + |\mathcal{E}_k|)$. Notably, $D_f$ and $D_k$ are much smaller than $D$, resulting in low computational overhead for large-scale graphs.

\section{Experiments}
In this section, we conduct extensive experiments to evaluate the effectiveness of ETR. The experiments aim to answer the following research questions: \textbf{RQ1}: How does ETR perform in terms of model utility, unlearning efficiency, and unlearning efficacy? 
\textbf{RQ2}: How does ETR perform on large-scale graphs? \textbf{RQ3}: How do hyperparameters influence ETR's performance? \textbf{RQ4}: How do different strategies in ETR contribute to its effectiveness?  

\subsection{Experimental Setup}
\subsubsection{Datasets} 
We conduct experiments on PubMed, CiteSeer, Cora~\cite{DBLP:conf/icml/YangCS16}, CS, Physics~\cite{DBLP:journals/corr/abs-1811-05868}, ogbn-arxiv, and ogbn-products~\cite{DBLP:conf/nips/HuFZDRLCL20}. The statistics and detailed descriptions of these datasets are provided in the Appendix.

\subsubsection{Baselines}
We compare ETR with various baselines, including Retrain, GraphEraser-BEKM, GraphEraser-BLPA~\cite{DBLP:conf/ccs/Chen000H022}, GIF~\cite{DBLP:conf/www/WuYQS0023}, GUIDE-SR, GUIDE-Fast~\cite{DBLP:conf/uss/WangH023}, GNNDelete~\cite{DBLP:conf/iclr/ChengDHAZ23}, and MEGU~\cite{DBLP:conf/aaai/LiZWZLW24}. Detailed descriptions are provided in the Appendix.

\subsubsection{Settings}
Following GIF, we partition each graph into a training subgraph comprising training nodes and a test subgraph containing the remaining nodes. For PubMed, CiteSeer, Cora, CS, and Physics, we randomly allocate 90\% of the nodes to the training set. For ogbn-arxiv and ogbn-products, we use the splits provided in~\cite{DBLP:conf/nips/HuFZDRLCL20}. The unlearning ratio is set to 5\% for all tasks. We use two-layer GCN and GAT models with a hidden state dimension of 256, training for 100 epochs. We tune the learning rates within the range of [1e-4, 1e-1] and the weight decay within [1e-7, 1e-2] for different datasets. All experiments are conducted on a single Nvidia A40 GPU. Each experiment is run ten times, and we report the average value and standard deviation.

\iffalse
\subsubsection{Evaluation Metrics}
We evaluate the performance of ETR in terms of model utility, unlearning efficiency, and unlearning efficacy as follows:

\begin{itemize}
  \item Model Utility: We evaluate the F1 score for the node classification task on the remaining dataset.
  \item Unlearning Efficiency: We evaluate the runtime and memory overhead of the unlearning methods.
  \item Unlearning Efficacy: We assess unlearning efficacy through both direct and indirect methods. The first method measures the root mean square distance between the parameters of the unlearned model and those of the retrained model. The second method evaluates the effectiveness of forgetting adversarial data.
\end{itemize}
\fi

\subsection{Evaluation of Model Utility (RQ1)}
We evaluate the F1 score for node classification on the remaining dataset. The results for node unlearning are shown in \tablename~\ref{tab:result}, while the results for edge and feature unlearning are provided in the Appendix. It can be observed that ETR consistently performs best in most cases, which can be attributed to the Rectify strategy that enhances model performance on the remaining dataset. Additionally, BEKM and BLPA exhibit suboptimal performance, consistent with the results reported in their original papers, likely due to disruptions in graph structure resulting from graph partitioning. GUIDE mitigates this problem, leading to improved performance. Among the baselines, MEGU achieves the best performance, which can be attributed to its training of the model’s predictive objective during unlearning.

\begin{table}[]
\centering
\setlength{\tabcolsep}{1mm}
\small
% \resizebox{\linewidth}{!}{
\begin{tabular}{cc|ccccc}
\hline
Bone             & Method  & PubMed        & CiteSeer      & Cora          & CS            & Physics       \\ \hline
\multirow{8}{*}{GCN} & Retrain & 1.57          & 1.52          & 1.30           & 6.13          & 14.62         \\
                     & BEKM    & 0.80           & 0.60           & 0.45          & 5.87          & 14.87         \\
                     & BLPA    & \textbf{0.45} & \textbf{0.48} & \textbf{0.38} & \textbf{1.44} & 2.89          \\
                     & GIF     & 1.70           & 1.00             & 0.61          & 11.18         & 37.63         \\
                     & SR      & 1.48          & 1.23          & 1.20           & 2.09          & 4.97          \\
                     & Fast    & 1.44          & 1.23          & 1.20           & 2.36          & 5.10           \\
                     & GNNDelete    & 1.64          & 0.87          & 1.22          & 2.76          & 5.25          \\
                     & MEGU    & 0.68          & 0.52          & 0.41          & 2.07          & 4.52          \\
                     & ETR     & 1.18          & 1.21          & 1.17          & 1.61          & \textbf{2.43} \\ \hline
\multirow{8}{*}{GAT} & Retrain & 1.77          & 1.52          & 1.30           & 6.17          & 15.28         \\
                     & BEKM    & 0.78          & 0.60           & 0.45          & 5.87          & 14.86         \\
                     & BLPA    & \textbf{0.48} & \textbf{0.48} & \textbf{0.40}  & 1.60           & 3.53          \\
                     & GIF     & 2.88          & 0.80           & 0.72          & 5.89          & 14.3          \\
                     & SR      & 1.52          & 1.23          & 1.20           & 2.20           & 5.21          \\
                     & Fast    & 1.52          & 1.23          & 1.20           & 2.41          & 5.32          \\
                     & GNNDelete    & 1.90          & 1.31          & 1.24          & 3.05          & 6.10          \\
                     & MEGU    & 0.94          & 0.57          & 0.45          & 3.62          & 7.64          \\
                     & ETR     & 1.20           & 1.21          & 1.17          & \textbf{1.54} & \textbf{2.30}  \\ \hline
\end{tabular}% }
\caption{The memory overhead (GB) of different methods.}
\label{tab:space}
\end{table}

\begin{table}[!t]
\centering
\small
\setlength{\tabcolsep}{1mm}
% \resizebox{\linewidth}{!}{
\begin{tabular}{cc|ccccc}
\hline
Backbone             & Method & PubMed          & CiteSeer         & Cora             & CS               & Physics      \\ \hline
\multirow{7}{*}{GCN} & BEKM   & 5e-2           & 2e-3           & 8e-3           & 2e-2            & 2e-2            \\
                     & BLPA  & 5e-2           & 2e-3           & 8e-3           & 2e-2            & 1e-2            \\
                     & GIF    & 7e-2           & 4e-3           & 2e-2            & 3e-2            & 1e-2            \\
                     & SR     & 9e-2           & 8e-3           & 2e-2            & 8e-3           & 2e-3             \\
                     & Fast   & 8e-2           & 9e-3           & 2e-2            & 8e-3           & 2e-3             \\
                     & MEGU   & \textbf{1e-2} & 3e-3           & 3e-2            & 1e-3           & 7e-4           \\
                     & ETR    & 3e-2           & \textbf{5e-4} & \textbf{9e-4} & \textbf{4e-4} & \textbf{2e-4}  \\ \hline
\multirow{7}{*}{GAT} & BEKM   & \textbf{9e-3}          & 1e-2            & 2e-2            & 9e-3           & 5e-4            \\
                     & BLPA  & \textbf{9e-3} & 1e-2            & 2e-2            & 8e-3           & 5e-4            \\
                     & GIF    & 6e-2           & 7e-3           & 2e-2            & 2e-3           & 3e-2            \\
                     & SR     & 7e-2           & 1e-2            & 1e-2            & 4e-3           & 7e-2             \\
                     & Fast   & 7e-2           & 1e-2            & 1e-2           & 4e-3           & 7e-2             \\
                     & MEGU   & 9e-2           & 4e-2            & 3e-3           & 2e-2            & 1e-2 \\
                     & ETR    & 3e-2           & \textbf{1e-3}  & \textbf{2e-3}  & \textbf{4e-4} & \textbf{2e-4}           \\ \hline
\end{tabular}% }
\caption{The root mean square distance between the parameters and those of Retrain.}
\label{tab:diff}
\end{table}

\subsection{Evaluation of Unlearning Efficiency (RQ1)}
The runtime results are presented in \tablename~\ref{tab:result}. ETR consistently outperforms all baselines, which can be attributed to its training-free nature. ETR reduces runtime by thousands of times compared to Retrain, demonstrating its superior efficiency. While GraphEraser and GUIDE are also faster than Retrain in most cases, they still require considerable time for unlearning due to the need for retraining multiple models. Although GIF and MEGU improve runtime efficiency compared to other baselines, they still fall short of ETR. Specifically, ETR reduces runtime by 38 times compared to GIF and 12 times compared to MEGU.

The memory overhead results are shown in \tablename~\ref{tab:space}. ETR performs comparably to the baselines on small datasets and outperforms them on larger datasets (Photo and Computer). This is because memory overhead on small datasets is primarily driven by model parameters, limiting the advantages of ETR. However, ETR reduces memory overhead on large datasets by not requiring the entire training dataset, decreasing it by 3.4 times compared to Retrain. Among the baselines, BLPA performs best due to its use of graph partitioning, which reduces memory overhead but significantly degrades model utility. Overall, ETR achieves an excellent trade-off between unlearning efficiency and model utility.

\begin{figure}[!t]
    \centering
    \includegraphics[width=0.8\linewidth]{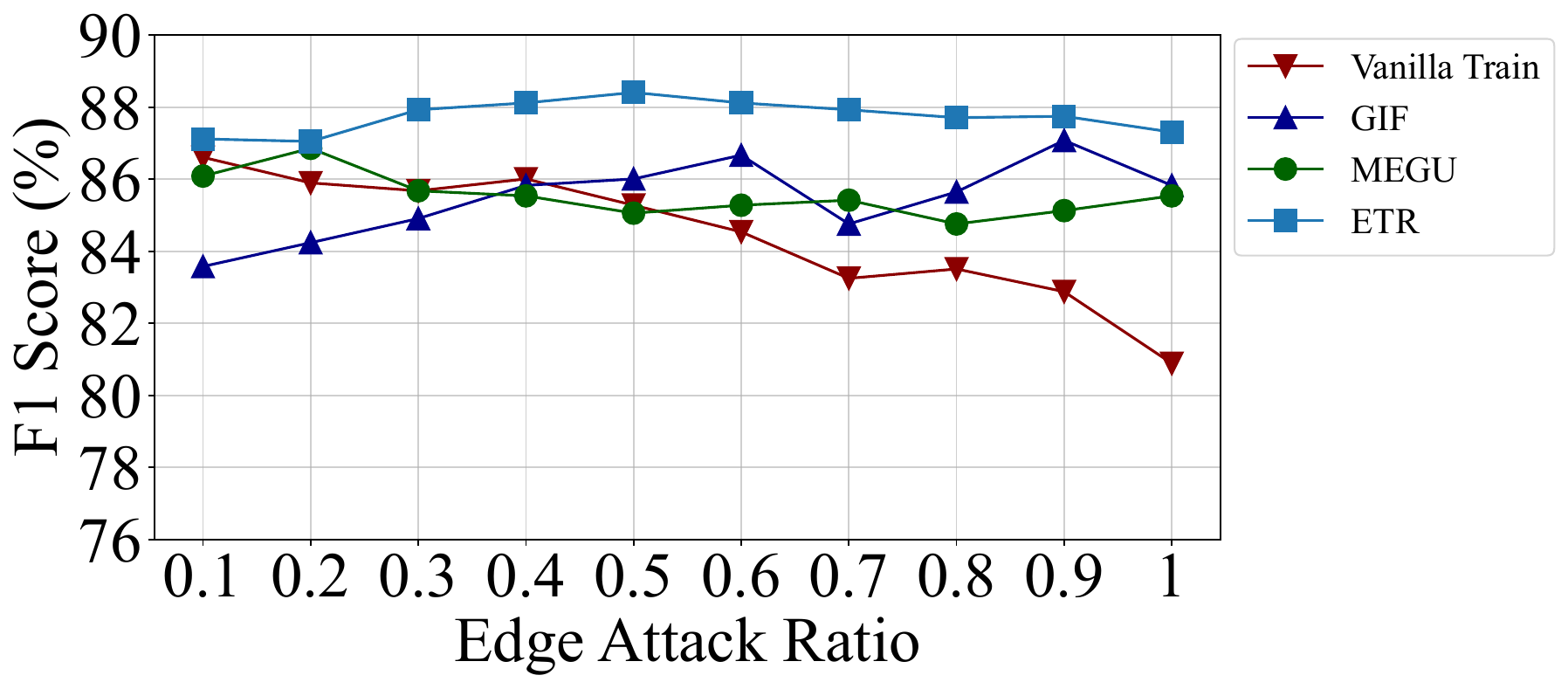}
    \caption{Adversarial edge unlearning on the Cora dataset.}
    % \vspace{-10pt}
    \label{fig:attack}
\end{figure}

\begin{table}[]
\centering
\small
\setlength{\tabcolsep}{0.8mm}
% \resizebox{\linewidth}{!}{
\begin{tabular}{cc|cccccccc}
\hline
\multirow{2}{*}{Bone} & \multirow{2}{*}{Method} & \multicolumn{4}{c}{ogbn-arxiv}                                         & \multicolumn{4}{c}{ogbn-products}                                     \\ \cmidrule(l){3-6} \cmidrule(l){7-10}
                          &                         & F1                  & T              & S             & D            & F1                  & T             & S             & D            \\ \hline
\multirow{6}{*}{GCN}      & Retrain                 & 67.5          & 152         & 3.7          & -               & 73.6          & 519        & 11.2         & -               \\
                          & BEKM                    & 64.8          & 1290        & 2.9          & 8e-1            & OM                 & OM           & OM           & OM             \\
                          & BLPA                    & 63.9          & 1323        & \textbf{0.9} & 8e-1            & OM                 & OM           & OM           & OM             \\
                          & GIF                     & 62.1          & 7           & 19.7         & 1            & OM                 & OM           & OM           & OM             \\
                          & GNNDelete                    & 51.7          & 34           & 8.2          & -           & OM          & OM          & OM         & OM           \\
                          & MEGU                    & 65.6          & 2           & 5.9          & 4e-2           & 71.3          & 7          & 26.8         & 6e-2           \\
                          & ETR                     & \textbf{66.4} & \textbf{0.08} & 1.2          & \textbf{2e-2}  & \textbf{73.8} & \textbf{0.5} & \textbf{1.2} & \textbf{9e-3} \\ \hline
\multirow{6}{*}{GAT}      & Retrain                 & 67.3 & 217.9         & 6          & -               & 73.9          & 709        & 11.2         & -               \\
                          & BEKM                    & 65.3          & 1804        & 3.0          & 6e-1            & OM                 & OM           & OM           & OM             \\
                          & BLPA                    & 64.7          & 1658        & 2.5          & 6e-1            & OM                 & OM           & OM           & OM             \\
                          & GIF                     & 66.1          & 12          & 34.5         & 5e-1            & OM                 & OM           & OM           & OM             \\
                          & GNNDelete                    & 49.6          & 44           & 10.1          & -           & OM          & OM          & OM         & OM           \\
                          & MEGU                    & 58.5          & 3           & 11.2         & \textbf{4e-3} & 63.7          & 8          & 26.9         & 5e-2           \\
                          & ETR                     & \textbf{67.8} & \textbf{0.1}  & \textbf{1}  & 2e-2           & \textbf{73.9} & \textbf{0.7} & \textbf{1.3} & \textbf{2e-2}  \\ \hline
\end{tabular}% }
\caption{Performance comparison on large-scale graphs, where S denotes memory overhead, D indicates parameter distance, and ``OM" represents ``out of memory". \eat{\han{what are the meanings of 'S' and 'D' in the table?}}}
\label{tab:large}
\end{table}

\subsection{Evaluation of Unlearning Efficacy (RQ1)}
We assess unlearning efficacy using both direct and indirect evaluation methods. For direct evaluation, we measure the root mean square distance between the parameters of the unlearned model and those of the retrained model. As shown in \tablename~\ref{tab:diff}, ETR consistently performs best in most cases, indicating that its parameters are closest to those obtained by retraining from scratch, thereby demonstrating ETR's effectiveness in achieving graph unlearning.

For indirect evaluation, we follow the approach used by GIF and MEGU, adding adversarial edges between nodes of different categories in the training graph. These edges serve as unlearning targets, and we then assess the utility of the unlearned model. As shown in \figurename~\ref{fig:attack}, the performance of the vanilla trained model declines with increasing attack ratios, while GIF, MEGU, and ETR maintain or even improve performance. This demonstrates the effectiveness of these methods. Notably, ETR consistently outperforms the other baselines, highlighting its strong unlearning efficacy.

\subsection{Evaluation on Large Scale Graphs (RQ2)}
We conduct experiments on the ogbn-arxiv and ogbn-products datasets to assess ETR's effectiveness and efficiency on large-scale graphs, with the results shown in \tablename~\ref{tab:large}. GUIDE encounters out-of-memory issues on both datasets, while GraphEraser and GIF also experience out-of-memory issues on the ogbn-products dataset. ETR demonstrates superior performance in terms of model utility, unlearning efficiency, and unlearning efficacy on large-scale graphs. Specifically, ETR shows an average utility improvement of 9.13\% over MEGU, reduces runtime by thousands of times compared to Retrain, and decreases memory overhead by 7 times compared to Retrain and 15 times compared to MEGU. In terms of unlearning efficacy, ETR performs best in most cases. Overall, ETR achieves better performance and efficiency compared to other methods.

\begin{figure}[!t]
    \centering
    \includegraphics[width=\linewidth]{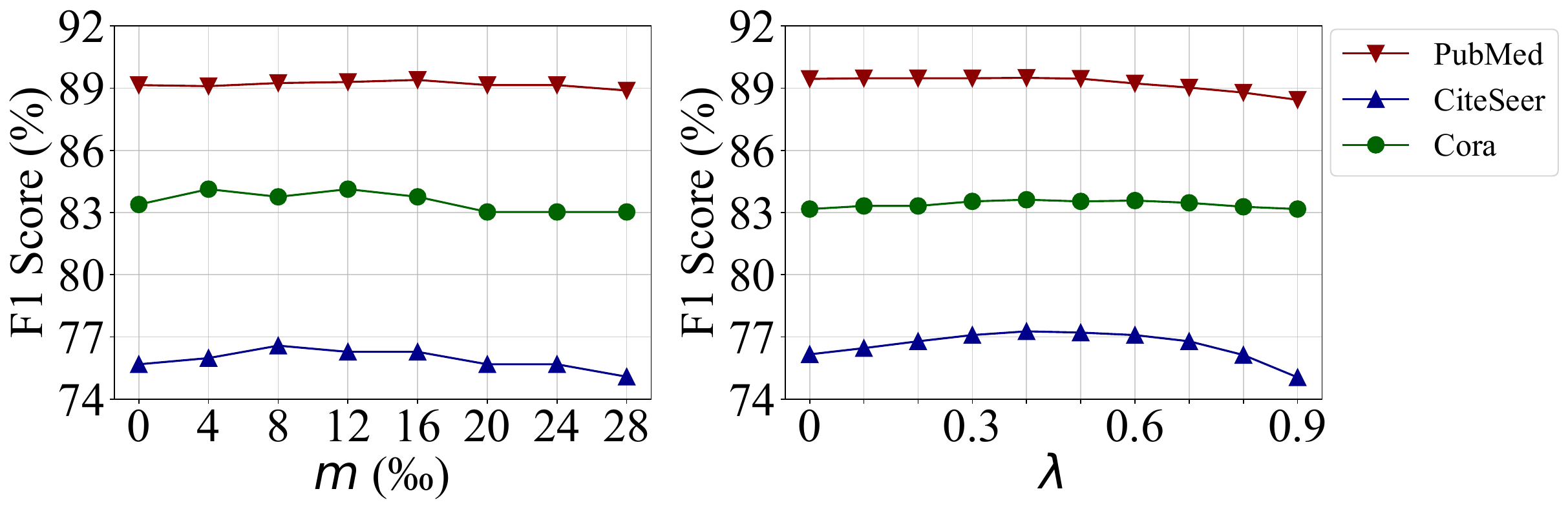}
    \caption{Performance with varying hyperparameters.}
    % \vspace{-10pt}
    \label{fig:para}
\end{figure}

\subsection{Hyperparameter Analysis (RQ3)}
The hyperparameters $m$ and $\lambda$ are crucial for ETR’s performance. The value of $m$ determines how much the parameters are modified, thereby influencing the extent of model forgetting. The value of $\lambda$ controls the extent of model rectification. In this section, we investigate their impact on ETR's performance using the PubMed, CiteSeer, and Cora datasets.

We investigate the effect of different $m$ values on the Erase strategy, with results shown in \figurename~\ref{fig:para}. It can be observed that increasing $m$ generally improves model performance, but performance declines if $m$ becomes too large. This is because larger $m$ values help the model forget unlearned samples effectively, while excessively large values can also lead to the loss of essential knowledge. Typically, $m$ values between 8‰ and 16‰ provide the best results for the Erase strategy. Notably, the model performs better compared to $m=0$, confirming that the Erase strategy effectively forgets unlearned samples and supports Theorem~\ref{theorem}.

We tune $\lambda$ within the range of [0, 0.9] to investigate how it affects ETR's performance, with results shown in \figurename~\ref{fig:para}. It can be observed that ETR's performance generally improves as $\lambda$ increases. However, performance declines if $\lambda$ becomes too large. This is because larger $\lambda$ values aid in model rectification, while excessively large values can also lead to excessive rectification. Optimal performance for ETR is typically achieved with $\lambda$ values between 0.3 and 0.5.

\begin{table}[!t]
\centering
\small
\setlength{\tabcolsep}{1.5mm}
% \resizebox{0.9\linewidth}{!}{
\begin{tabular}{c|ccccc}
\hline
Method  & PubMed         & CiteSeer       & Cora           & CS             & Physics     \\ \hline
w/ mask & 88.91          & 77.00          & 82.51          & 92.44          & \textbf{95.80}          \\
w/o mes & 89.30          & 76.91          & 82.44          & 92.53          & 95.70          \\
ETR     & \textbf{89.48} & \textbf{77.09} & \textbf{83.54} & \textbf{92.72} & 95.71 \\ \hline
\end{tabular}% }
\caption{Performance comparison between model variants.}
\label{tab:ablation}
\end{table}

\begin{table}[!t]
\centering
\small
\setlength{\tabcolsep}{1.5mm}
\begin{tabular}{c|cccccc}
\hline
Metric & PubMed   & CiteSeer & Cora     & CS       & Physics \\ \hline
AD     & 5.7E-06  & 8.1E-07  & 1.7E-06  & 5.3E-06  & 1.6E-04   \\
RD     & 0.23     & 0.09     & 0.07     & 0.36     & 0.22       \\ \hline
\end{tabular}
\caption{The difference between the approximate gradient and the true gradient.}
\label{tab:appro}
\end{table}

\subsection{Ablation Studies (RQ4)}
This section investigates the effectiveness of different strategies in ETR. The Erase strategy is crucial for model forgetting, while Rectify is critical for maintaining model utility. Results in \figurename~\ref{fig:para} indicate that increasing $m$ improves performance compared to $m=0$, demonstrating the effectiveness of the Erase strategy in forgetting unlearned samples. Similarly, increasing $\lambda$ improves performance compared to $\lambda=0$, indicating that the Rectify strategy effectively enhances the model performance on the remaining dataset.

To further investigate the effectiveness of the Erase strategy, we compare ETR with two model variants, as shown in \tablename~\ref{tab:ablation}. Specifically, ``w/ mask" refers to using the masking strategy, while ``w/o mes" indicates not considering the impact of message propagation. The results show that ETR outperforms both variants in most cases, except for a slight decrease on the Physics dataset compared to ``w/ mask". This demonstrates that parameter editing is more effective than just applying a mask and emphasizes the importance of considering message propagation in graph unlearning.

In the Rectify stage, we use the induced graph of unlearned samples to approximate the gradient of the model on the remaining dataset. We measure the difference of this approximation using Absolute Difference (AD) and Relative Difference (RD), as shown in \tablename~\ref{tab:appro}. The results show that both differences are minimal, indicating that the gradient approximation ensures both efficiency and accuracy.

\section{Conclusion}
In this paper, we investigate cost-effective graph unlearning and proposes the ETR method, which removes the influence of unlearned samples while preserving model performance on the remaining data. In the Erase stage, we use the Fisher Information Matrix to identify parameters crucial for unlearned samples and their impact on connected samples, then modify them to forget both the unlearned samples and their influence. In the Rectify stage, we leverage the induced graph of unlearned samples to approximate the model's gradient on the remaining data, which is then used to enhance model performance. Extensive experiments demonstrate that ETR achieves strong performance in both efficiency and effectiveness. Future work will explore unlearning partial attributes in multi-attribute samples and partial objectives in multi-task settings~\cite{DBLP:conf/nips/LiuT15, DBLP:journals/jmlr/LiuTM17, DBLP:journals/pami/LiuXTZ19}, where strong dependencies between different attributes and different tasks pose challenges, especially in graph-based scenarios, where there are also strong dependencies between samples. 

\section*{Acknowledgements}
This work was supported by the National Natural Science Foundation of China (Grant No. 62276277, 92370204), the Guangdong Basic and Applied Basic Research Foundation (Grant No. 2022B1515120059), the National Key R\&D Program of China (Grant No. 2023YFF0725004), the Guangzhou-HKUST(GZ) Joint Funding Program (Grant No. 2023A03J0008), and the Education Bureau of Guangzhou Municipality.

\bibliography{reference}

\clearpage

\appendix
\section{Theoretical Analysis}
Prior to proving Theorem~\ref{theorem}, we first introduce some notation that will be used in the proof. Considering the architectural differences among various GNNs, for the sake of simplicity in our analysis, we consider the following general two-layer GNN architecture:
\begin{align}
    H_1&=\sigma \left( PXW_0 \right) \\
    H_2&=PH_1W_1 \\
    Z&=softmax \left( H_2 \right) 
\end{align}
where $P$ represents the message transformation matrix associated with the adjacency matrix, $\sigma(\cdot)$ denotes the activation function, $W_0$ and $W_1$ represent parameter matrices. In the node classification task, the loss function is given by:
\begin{align}
    L=-\small{\frac{1}{\left| D \right|}}\sum_{i\in D}{\sum_{j=1}^C{Y_{i,j}\log \left( Z_{i,j} \right)}}.
\end{align}

Based on the definition of the Fisher Information Matrix, the following properties are established:
\begin{small}
    \begin{align}
    F_D&=\frac{1}{\left| D \right|}\sum_{i=1}^{\left| D \right|}{\nabla _w\log p\left( y_i|x_i,w \right) \nabla _w\log p\left( y_i|x_i,w \right) ^T} \\
    &=\frac{1}{\left| D \right|}\sum_{i\in D_f}{\nabla _w\log p\left( y_i|x_i,w \right) \nabla _w\log p\left( y_i|x_i,w \right) ^T} \\
    &+\frac{1}{\left| D \right|}\sum_{i\in D_r}{\nabla _w\log p\left( y_j|x_j,w \right) \nabla _w\log p\left( y_j|x_j,w \right) ^T} \\
    &=\frac{\left| D_f \right|}{\left| D \right|}F_{D_f}+\frac{\left| D_r \right|}{\left| D \right|}F_{D_r}
\end{align}
\end{small}

We denote the vectorization of the matrix $W$ as $vec(W)$ and define $\omega = vec(W_0, W_1) = concat(vec(W_0), vec(W_1))$, where $concat(\cdot, \cdot)$ denotes the concatenation of two vectors. We now present the proof of Theorem~\ref{theorem}.

\begin{proof}[Proof of Theorem~\ref{theorem}]
First, we compute the gradients of the backpropagation using the node classification loss function as follows:
\begin{small}
    \begin{align}
    \frac{\partial L}{\partial W_1}&=H_{1}^{T}P^T\left( Z-Y \right) \\
    \frac{\partial L}{\partial W_0}&=X^TP^T\left\{ \left[ \left( Z-Y \right) W_{1}^{T}P^T \right] \odot \sigma' \left( PXW_0 \right) \right\} 
\end{align}
\end{small}
According to the definition of the Fisher Information Matrix, we have: $F=vec\left( \frac{\partial L}{\partial W_1},\frac{\partial L}{\partial W_0} \right) vec\left( \frac{\partial L}{\partial W_1},\frac{\partial L}{\partial W_0} \right) ^T$. For the sake of simplicity, we decompose $F$ by analyzing $W_0$ and $W_1$ separately. Define $F_0=vec\left( \frac{\partial L}{\partial W_0} \right) vec\left( \frac{\partial L}{\partial W_0} \right) ^T$ and $F_1=vec\left( \frac{\partial L}{\partial W_1} \right) vec\left( \frac{\partial L}{\partial W_1} \right) ^T$. Accordingly, we have: $F_1vec\left( W_1 \right) =tr\left( H_2\left( Z-Y \right) \right) vec\left( \frac{\partial L}{\partial W_1} \right) $, and $F_0vec\left( W_0 \right) =tr\left( W_0\frac{\partial L}{\partial W_0} \right) vec\left( \frac{\partial L}{\partial W_0} \right) $. We denote $b\coloneqq F\omega \approx \left[ F_0W_0;F_1W_1 \right] $

According to Formulation (\ref{mask}), we have: 
\begin{align}
    &\frac{1}{\left| \omega \right|}\sum_j{\left( \omega _{r,j}^{*}-\,\,\widehat{\omega}_{r,j} \right)}^2 \\
    =&\frac{1}{\left| \omega \right|}\left( \sum_{j\in M}{{\omega _{r,j}^{*}}^2}+\sum_{j\notin M}{\left( \omega _{r,j}^{*}-\,\,\omega _{j}^{*} \right) ^2} \right) 
    \label{2part}
\end{align}
Using $diag(F)$ to approximate the Fisher information matrix $F$, the first term of Formulation (\ref{2part}) is given by:
\begin{align}
    \sum_{j\in M}{{\omega _{r,j}^{*}}^2}=\sum_{j\in M}{\frac{b_{r,j}^{2}}{F_{D_r,jj}^{2}}}\leqslant c_1\sum_{j\in M}{\frac{1}{F_{D_r,jj}^{2}}}
\end{align}
where $c_1=\underset{j}{\max}\ b_{r,j}^{2}$. And the second term of Formulation (24) is given by:
\begin{small}
    \begin{align}
    &\sum_{j\notin M}{\left( \omega _{r,j}^{*}-\,\,\omega _{j}^{*} \right) ^2}\\
    =&\sum_{j\notin M}{\left( \frac{b_j}{F_{D,jj}}-\,\,\frac{b_{r,j}}{F_{D_r,jj}} \right) ^2} \\
    =&\sum_{j\notin M}{\left( \frac{b_jF_{D_r,jj}-b_{r,j}F_{D,jj}}{F_{D,jj}F_{D_r,jj}} \right) ^2} \\
    =&\sum_{j\notin M}{\left( \frac{b_jF_{D_r,jj}-b_{r,j}\frac{\left| D_f \right|}{\left| D \right|}F_{D_f,jj}-b_{r,j}\frac{\left| D_r \right|}{\left| D \right|}F_{D_r,jj}}{F_{D,jj}F_{D_r,jj}} \right) ^2} \\
    =&\sum_{j\notin M}{\left( \left( b_j-b_{r,j}\frac{\left| D_r \right|}{\left| D \right|} \right) \frac{1}{F_{D,jj}}-b_{r,j}\frac{\left| D_f \right|}{\left| D \right|}\frac{F_{D_f,jj}}{F_{D,jj}F_{D_r,jj}} \right) ^2} \\
    \leqslant &\sum_{j\notin M}{\left( b_j-b_{r,j}\frac{\left| D_r \right|}{\left| D \right|} \right) ^2}\frac{2}{F_{D,jj}^{2}} \\
    +&\sum_{j\notin M}{\left( b_{r,j}\frac{\left| D_f \right|}{\left| D \right|} \right) ^2}\frac{F_{D_f,jj}^{2}}{F_{D,jj}^{2}F_{D_r,jj}^{2}}\\
    \leqslant& c_2\sum_{j\notin M}{\frac{F_{D_f,jj}^{2}}{F_{j,j}^{2}F_{r,jj}^{2}}}+c_3
    \end{align}
\end{small}
where $c_2=\underset{j}{\max}\left( b_{r,j}\frac{\left| D_f \right|}{\left| D \right|} \right) ^2$, and $c_3=\sum_j{\left( b_j-b_{r,j}\frac{\left| D_r \right|}{\left| D \right|} \right) ^2}\frac{2}{F_{D,jj}^{2}}$.

Therefore, we can obtain:
\begin{small}
    \begin{align}
    &\frac{1}{\left| \omega \right|}\sum_j{\left( \omega _{r,j}^{*}-\,\,\widehat{\omega}_{r,j} \right)}^2 \\
    =&\frac{1}{\left| \omega \right|}\left( \sum_{j\in M}{{\omega _{r,j}^{*}}^2}+\sum_{j\notin M}{\left( \omega _{r,j}^{*}-\,\,\omega _{j}^{*} \right) ^2} \right) \\
    \leqslant& \frac{1}{\left| \omega \right|}\left( c_1\sum_{j\in M}{\frac{1}{F_{D_r,jj}^{2}}}+c_2\sum_{j\notin M}{\frac{F_{D_f,jj}^{2}}{F_{D,jj}^{2}F_{D_r,jj}^{2}}}+c_3 \right) 
\end{align}
\end{small}
\end{proof}

\section{Datasets}
\label{appendix:dataset}
\subsection{Dataset Statistics}

\begin{table}[h]
\resizebox{\linewidth}{!}{
\begin{tabular}{@{}c|cccc@{}}
\toprule
Dataset  & \#Nodes & \#Edges & \#Features & \#Classes \\ \midrule
PubMed   & 19,717   & 88,651   & 500        & 3         \\
CiteSeer & 3,327    & 9,228    & 3,703      & 6         \\
Cora     & 2,708    & 10,556   & 1,433      & 7         \\
CS       & 18,333   & 163,788  & 6,805      & 15         \\
Physics  & 34,493   & 495,924  & 8,415      & 5        \\
ogbn-arxiv   & 169,343 & 1,166,243  & 128   & 40        \\ 
ogbn-products   & 2,449,029 & 61,859,140  & 100   & 47        \\ \bottomrule
\end{tabular}}
\caption{Statistics of datasets.}
\label{tab:stat}
\end{table}

\subsection{Dataset Descriptions}
\begin{itemize}
  \item \textbf{PubMed, CiteSeer, and Cora}. PubMed, CiteSeer, and Cora are three standard citation network benchmark datasets. In these datasets, the nodes correspond to research papers, and the edges represent the citations between papers. The node features are derived from the bag-of-words representation of the papers, while the node labels indicate the academic topics of the papers.
  \item \textbf{CS and Physics}. CS and Physics are two co-authorship networks focused on research papers. In these datasets, nodes represent authors of research papers, while edges denote the co-authorship relationships among them. Node features derive from the keywords associated with each author's publications, and node labels indicate the most active fields of study for each author.
  \item \textbf{ogbn-arxiv}. ogbn-arxiv is a large-scale citation network comprising Computer Science papers from arXiv. In this dataset, nodes represent individual arXiv papers, while edges signify the citation relationships between them. Node features are generated by averaging the embeddings of the words found in the title and abstract of each paper. Node labels denote the subject areas of the papers.
  \item \textbf{ogbn-products}. ogbn-products is a large-scale co-purchasing network of Amazon products. In this dataset, nodes represent individual products sold on Amazon, while edges indicate pairs of products that were purchased together. Node features are derived by extracting bag-of-words features from product descriptions. Node labels denote the product categories.
\end{itemize}

\section{Baselines}
\label{appendix:baseline}
\begin{itemize}
    \item \textbf{Retrain}: Retrain is an exact graph unlearning method that entails retraining the model from scratch on the remaining dataset upon receiving an unlearning request.
    \item \textbf{GraphEraser}: GraphEraser is the first approximate graph unlearning method, adhering to the SISA paradigm. It partitions the training graph into multiple shards, trains a model on each shard, and subsequently aggregates the various submodels. Upon receiving an unlearning request, only the affected submodels need to be retrained. GraphEraser-BEKM and GraphEraser-BLPA denote the application of different graph partitioning methods.
    \item \textbf{GIF}: GIF utilizes influence functions to approximate the parameter changes resulting from an unlearning request. It subsequently updates the model parameters based on these estimated changes to facilitate graph unlearning.
    \item \textbf{GUIDE}: GUIDE is specifically designed for inductive graph unlearning and follows the SISA paradigm. GUIDE takes into account fairness and balance in graph partitioning and addresses the disruptions caused to the graph structure. GUIDE-Fast and GUIDE-SR represent the application of different graph partitioning methods.
    \item \textbf{GNNDelete}: GNNDelete freezes the model parameters and introduces a learnable deletion operator to ensure that the influence of deleted elements is removed while preserving the remaining knowledge of the model.
    \item \textbf{MEGU}: MEGU integrates the model's predictive capability and forgetting ability within a unified optimization framework, thereby facilitating mutual benefits between these two objectives.
\end{itemize}

\begin{algorithm}[!t]
\caption{ETR for Edge and Feature Unlearning}
\label{alg2}
\begin{algorithmic}[1]
\Require  
    Influenced dataset $D_i$; GNN $f_\mathcal{G}$; Optimal parameters $\omega^{*}$; Gradient $\nabla _{\omega^{*}}L_D$; Hyperparameters $m$ and $\lambda$.
\Statex \textit{\textbf{Erase}}
    \State Calculate the gradients $\nabla _{\omega^{*}}L_{D_i}$.
    \State Calculate the FIM $F_{D}$, $F_{D_i}$ via $(\ref{FIM})$.
    \State Obtain $\gamma=top\mbox{-}k\%(\frac{F_{D_i}}{F_{D}})$.
    \For{$j \ in\  range \ |\omega|$}
        \State Obtain $\alpha _j=\frac{F_{D,jj}}{F_{D_i,jj}}$
        \If{$F_{D_i,jj} > \gamma F_{D,jj}$}
            \State $\widehat{\omega}_{r,j}=\alpha _j\gamma \omega _j^*$
        \Else
            \State $\widehat{\omega}_{r,j}=\omega _j^*$
        \EndIf
    \EndFor
    
\Statex \textit{\textbf{Rectify}}
    \State Calculate the gradient $\nabla _{\widehat{\omega}_{r}}L_{D_k}$.
    \State Obtain the gradient $\nabla _{\widehat{\omega}_{r}}L_{D_r}$ via $(\ref{gradient2})$.
    \State Obtain the parameters $\omega ^{\prime}$ via $(\ref{update})$.
\Ensure 
    $\omega ^{\prime}$.
\end{algorithmic}
\end{algorithm}

\begin{table*}[!t]
\centering
\begin{tabular}{cc|ccccc}
\hline
Backbone                   & Method  & PubMed     & CiteSeer   & Cora       & CS       & Physics         \\ \hline
\multirow{8}{*}{GraphSage} & Retrain & 88.94±0.31 & 74.02±0.43 & 77.64±1.56 & 93.09±0.26 & 96.54±0.11 \\
                           & BEKM    & 85.86±0.22 & 73.63±0.62 & 80.70±0.87 & 90.49±0.28 & 95.34±0.11 \\
                           & BLPA    & 85.15±0.26 & 73.45±0.78 & 72.95±0.65 & 90.05±0.27 & 94.71±0.16 \\
                           & GIF     & 86.96±0.38 & 75.91±0.77 & \underline{81.59±0.73} & 92.88±1.76 & 96.01±0.15 \\
                           & SR      & 86.58±0.28 & 74.04±1.50 & 78.41±1.12 & 91.31±0.44 & 94.88±0.14 \\
                           & Fast    & 86.70±0.32 & 74.97±0.78 & 78.60±1.11 & 91.18±0.52 & 94.93±0.07 \\
                           & MEGU    & \textbf{87.34±0.27} & \underline{75.67±0.94} & 80.22±0.50 & \underline{93.35±0.13} & \textbf{96.64±0.21} \\
                           & ETR     & \underline{87.22±0.26} & \textbf{76.64±0.69} & \textbf{83.10±0.95} & \textbf{93.86±0.18} & \underline{96.27±0.11} \\ \hline
\multirow{8}{*}{GIN}       & Retrain & 85.23±0.33 & 74.53±0.80 & 80.85±1.02 & 91.15±0.37 & 95.42±0.34 \\
                           & BEKM    & 83.56±0.37 & 68.89±3.65 & 78.89±2.42 & 87.11±0.31 & 93.83±0.25 \\
                           & BLPA    & 83.49±0.46 & 67.17±3.33 & 73.8±3.06  & 87.87±0.38 & 93.34±0.14 \\
                           & GIF     & 84.91±1.16 & 70.99±1.84 & 78.19±1.15 & 88.47±0.68 & \underline{94.83±0.23} \\
                           & SR      & 84.31±0.36 & 68.53±2.23 & 79.82±1.78 & 89.46±0.46 & 92.40±0.51 \\
                           & Fast    & 84.10±0.49 & 67.06±2.16 & 80.07±1.34 & 89.16±0.51 & 93.65±0.29 \\
                           & MEGU    & \underline{84.93±0.64} & \underline{74.68±1.03} & \textbf{81.66±0.80} & \underline{89.47±0.46} & 93.91±0.30 \\
                           & ETR     & \textbf{85.07±0.34} & \textbf{75.47±0.65} & \underline{80.92±1.07} & \textbf{90.20±0.46}          & \textbf{95.02±0.20} \\ \hline
\end{tabular}
\caption{Performance comparison in terms of model utility with GraphSage and GIN backbones.}
\label{tab:sage_gin}
\end{table*}

\begin{table*}[!t]
\centering
\begin{tabular}{cc|cccccc}
\hline
Backbone             & Method  & PubMed              & CiteSeer            & Cora                & CS                  & Physics             & ogbn-arxiv          \\ \hline
\multirow{8}{*}{GCN} & Retrain & 89.95±0.18          & 78.17±0.84          & 86.64±0.54          & 93.48±0.23          & 96.45±0.18          & 68.07±1.38          \\
                     & BEKM    & 78.90±0.48          & 59.85±1.70          & 65.68±3.01          & 87.62±0.64          & 94.66±0.37          & 64.90±0.25          \\
                     & BLPA    & 80.52±0.44          & 56.52±2.86          & 60.66±2.90          & 87.48±0.74          & 94.46±0.34          & 63.60±0.56          \\
                     & GIF     & 83.20±0.22          & 72.55±1.06          & 82.62±1.38          & 92.62±0.36          & 95.61±0.16          & 65.76±1.06          \\
                     & SR      & 87.36±0.28          & 78.26±0.79          & 84.46±0.93          & 92.14±0.14          & 94.91±0.12          & OOM                 \\
                     & Fast    & 87.36±0.24          & \underline{78.44±0.87}          & 84.65±0.89          & 92.14±0.17          & 94.95±0.12          & OOM                 \\
                     & MEGU    & \underline{87.81±0.23}          & 77.27±0.81          & \underline{86.53±0.80}          & \underline{93.23±0.40}          & \underline{96.01±0.11}          & \underline{66.97±1.67}          \\
                     & ETR     & \textbf{89.88±0.21} & \textbf{79.28±0.76} & \textbf{87.23±0.76} & \textbf{93.48±0.26} & \textbf{96.43±0.16} & \textbf{68.10±0.63} \\ \hline
\multirow{8}{*}{GAT} & Retrain & 88.60±0.26          & 77.66±0.91          & 84.13±1.44          & 93.08±0.17          & 96.65±0.11          & 69.03±0.60          \\
                     & BEKM    & 70.01±0.91          & 53.48±2.34          & 48.93±2.70          & 82.19±0.47          & 91.09±0.44          & 64.79±0.22          \\
                     & BLPA    & 71.26±1.17          & 51.59±2.86          & 49.89±6.71          & 81.80±0.62          & 90.83±0.27          & 64.37±0.26          \\
                     & GIF     & \underline{88.04±0.33}          & \underline{76.91±1.35}          & \underline{85.06±1.18}          & \textbf{93.59±0.22} & \underline{96.26±0.19}          & \underline{67.49±1.46}          \\
                     & SR      & 85.74±0.49          & 76.25±0.85          & 79.82±1.54          & 90.51±0.68          & 94.46±0.98          & OOM                 \\
                     & Fast    & 85.81±0.64          & 75.65±0.86          & 79.45±2.01          & 90.60±0.76          & 93.81±1.31          & OOM                 \\
                     & MEGU    & 84.24±0.40          & 75.44±0.88          & \textbf{86.61±1.11} & 92.32±0.31          & 95.09±0.19          & 59.69±1.69          \\
                     & ETR     & \textbf{88.21±0.48} & \textbf{77.87±0.93} & 84.32±0.91          & \underline{93.40±0.27}          & \textbf{96.61±0.17} & \textbf{69.40±0.46} \\ \hline
\end{tabular}
\caption{Performance comparison in terms of model utility in the edge unlearning task.}
\label{tab:edge_1}
\end{table*}

\section{ETR for Edge Unlearning and Feature Unlearning}
% In this paper, we employ the node unlearning task to illustrate the ETR method. The ETR approach is also applicable to edge unlearning and feature unlearning tasks. Specifically, f
For edge and feature unlearning, it is necessary to forget the impact of the unlearned edges or features on other nodes. We denote the subgraph affected by the unlearned edges or features as $D_i$. Consequently, we modify the parameters that are crucial for $D_i$ but not for the remaining dataset, as detailed in Algorithm~\ref{alg2}. Additionally, we approximate the gradient of the model on the remaining dataset as follows: 
\begin{align}
    \nabla _{\widehat{\omega }_r}L_{D_r} =\small{\frac{1}{\left| D_r \right|}}( &\left| D \right|\nabla _{\omega^{*}}L_D-\left| D_i \right|\nabla _{\omega^{*}}L_{D_i}+\\
    &\left| D_i \right|\nabla _{\widehat{\omega }_r}L_{D_i} )
    \label{gradient2}
\end{align}
Furthermore, we utilize this gradient to enhance the model performance on the remaining dataset, as demonstrated in Algorithm~\ref{alg2}.

\section{Additional Experimental Results}

To investigate ETR's performance with different backbones, we conduct experiments using the GraphSage and GIN architectures, with results shown in Table~\ref{tab:sage_gin}. The results demonstrate that ETR outperforms other baselines. We also evaluate ETR's performance on edge unlearning and feature unlearning tasks. The model's utility, unlearning efficiency, and unlearning efficacy for edge unlearning are presented in Tables~\ref{tab:edge_1},\ref{tab:edge_2}, and\ref{tab:edge_3}, respectively. Additionally, performance on feature unlearning tasks is detailed in Tables~\ref{tab:feature_1},\ref{tab:feature_2}, and\ref{tab:feature_3}.

\begin{table*}[!t]
\centering
\begin{tabular}{cc|cccccc}
\hline
Backbone             & Method  & PubMed              & CiteSeer            & Cora                & CS                  & Physics             & ogbn-arxiv          \\ \hline
\multirow{8}{*}{GCN} & Retrain & 89.96±0.26          & 78.35±0.62          & 87.16±0.42          & 93.59±0.13          & 96.55±0.11          & 68.59±1.27          \\
                     & BEKM    & 79.94±0.50          & 58.26±3.50          & 60.89±6.51          & 89.22±0.52          & 94.67±0.17          & 65.01±0.24          \\
                     & BLPA    & 81.60±0.59          & 54.17±3.23          & 61.33±2.85          & 88.79±0.33          & 94.58±0.24          & 63.87±0.33          \\
                     & GIF     & 84.60±0.26          & 73.54±0.48          & 84.24±1.00          & 93.06±0.27          & 95.63±0.13          & 65.52±1.30          \\
                     & SR      & 87.44±0.24          & 77.03±0.50          & 78.19±4.21          & 91.19±0.36          & 95.00±0.09          & OOM                 \\
                     & Fast    & 87.44±0.35          & 76.97±0.76          & 78.23±3.76          & 91.32±0.21          & 95.00±0.10          & OOM                 \\
                     & MEGU    & \underline{87.87±0.17}          & \underline{77.60±0.51}          & \underline{85.65±0.71}          & \underline{93.50±0.17}          & \underline{96.01±0.14}          & \underline{66.70±1.95}          \\
                     & EAR     & \textbf{90.17±0.24} & \textbf{78.98±0.71} & \textbf{87.08±0.43} & \textbf{93.63±0.22} & \textbf{96.54±0.11} & \textbf{68.27±0.92} \\ \hline
\multirow{8}{*}{GAT} & Retrain & 88.52±0.26          & 77.66±0.79          & 85.20±0.81          & 93.25±0.46          & 96.57±0.20          & 68.76±0.77          \\
                     & BEKM    & 69.31±1.15          & 52.76±2.92          & 59.45±3.08          & 82.52±0.70          & 91.15±0.41          & \underline{65.38±0.23}          \\
                     & BLPA    & 70.84±0.95          & 48.32±1.71          & 55.17±3.94          & 82.72±0.38          & 91.17±0.36          & 64.82±0.25          \\
                     & GIF     & \underline{88.11±0.37}          & \underline{77.33±1.30}          & \underline{85.35±1.43}          & \underline{93.04±0.22}         & \underline{96.58±0.18}          & 64.38±11.07         \\
                     & SR      & 85.65±0.30          & 75.35±0.87          & 73.91±3.47          & 82.55±3.33          & 91.73±1.89          & OOM                 \\
                     & Fast    & 85.34±0.28          & 75.83±0.73          & 73.25±4.07          & 82.57±4.09          & 90.36±2.70          & OOM                 \\
                     & MEGU    & 84.20±0.48          & 75.59±0.91          & \textbf{87.31±1.06} & 92.24±0.24          & 95.50±0.14          & 59.40±1.75          \\
                     & EAR     & \textbf{88.45±0.41} & \textbf{77.84±0.96} & 84.58±0.93          & \textbf{93.22±0.25} & \textbf{96.66±0.15} & \textbf{69.54±0.35} \\ \hline
\end{tabular}
\caption{Performance comparison in terms of model utility in the feature unlearning task.}
\label{tab:feature_1}
\end{table*}

\begin{table*}[!t]
\centering
\begin{tabular}{cc|cccccc}
\hline
Backbone             & Method  & PubMed         & CiteSeer       & Cora           & CS             & Physics        & ogbn-arxiv     \\ \hline
\multirow{8}{*}{GCN} & Retrain & 42.82s         & 15.64s         & 8.31s          & 357.49s        & 1057.06s       & 173.59s        \\
                     & BEKM    & 76.60s         & 12.92s         & 10.92s         & 66.88s         & 138.76s        & 1314.22s       \\
                     & BLPA    & 65.70s         & 13.16s         & 11.36s         & 69.25s         & 140.64s        & 1424.94s       \\
                     & GIF     & 0.40s          & 0.22s          & 0.21s          & 1.30s          & 1.92s          & 6.87s          \\
                     & SR      & 30.81s         & 8.71s          & 8.80s          & 52.74s         & 149.74s        & OOM            \\
                     & Fast    & 30.71s         & 8.65s          & 8.71s          & 52.15s         & 152.54s        & OOM            \\
                     & MEGU    & \underline{0.30s}          & \underline{0.20s}          & \underline{0.19s}          & \underline{0.53s}          & \underline{1.04s}          & \underline{2.77s}          \\
                     & ETR     & \textbf{0.01s} & \textbf{0.01s} & \textbf{0.01s} & \textbf{0.03s} & \textbf{0.05s} & \textbf{0.06s} \\ \hline
\multirow{8}{*}{GAT} & Retrain & 54.58s         & 18.28s         & 9.96s          & 369.58s        & 1079.54s       & 238.50s        \\
                     & BEKM    & 112.55s        & 20.24s         & 17.11s         & 94.77s         & 182.37s        & 1435.04s       \\
                     & BLPA    & 96.90s         & 20.47s         & 17.72s         & 97.58s         & 188.74s        & 1558.15s       \\
                     & GIF     & 1.11s          & 0.68s          & 0.64s          & 2.10s          & 5.10s          & 11.56s         \\
                     & SR      & 35.80s         & 12.97s         & 12.53s         & 56.25s         & 158.29s        & OOM            \\
                     & Fast    & 35.54s         & 12.22s         & 12.24s         & 58.30s         & 157.83s        & OOM            \\
                     & MEGU    & \underline{0.37s}          & \underline{0.23s}          & \underline{0.22s}          & \underline{0.66s}          & \underline{1.36s}          & \underline{3.64s}          \\
                     & ETR     & \textbf{0.02s} & \textbf{0.02s} & \textbf{0.02s} & \textbf{0.04s} & \textbf{0.06s} & \textbf{0.08s} \\ \hline
\end{tabular}
\caption{Performance comparison in terms of unlearning efficiency in the edge unlearning task.}
\label{tab:edge_2}
\end{table*}

\begin{table*}[!t]
\centering
\begin{tabular}{cc|cccccc}
\hline
Backbone             & Method  & PubMed         & CiteSeer       & Cora           & CS             & Physics        & ogbn-arxiv     \\ \hline
\multirow{8}{*}{GCN} & Retrain & 42.22s         & 15.60s         & 7.88s          & 322.21s        & 943.82s        & 168.91s        \\
                     & BEKM    & 71.08s         & 12.95s         & 11.52s         & 71.71s         & 142.22s        & 1327.08s       \\
                     & BLPA    & 68.97s         & 13.67s         & 11.63s         & 69.52s         & 143.42s        & 1320.60s       \\
                     & GIF     & 0.39s          & 0.22s          & \underline{0.18s}          & 1.30s          & 1.94s          & 6.88s          \\
                     & SR      & 30.34s         & 8.44s          & 8.73s          & 52.22s         & 148.24s        & OOM            \\
                     & Fast    & 30.92s         & 8.62s          & 8.35s          & 52.06s         & 147.64s        & OOM            \\
                     & MEGU    & \underline{0.31s}          & \underline{0.20s}          & 0.19s          & \underline{0.53s}          & \underline{1.05s}          & \underline{2.77s}          \\
                     & ETR     & \textbf{0.01s} & \textbf{0.01s} & \textbf{0.01s} & \textbf{0.03s} & \textbf{0.05s} & \textbf{0.06s} \\ \hline
\multirow{8}{*}{GAT} & Retrain & 53.17s         & 17.84s         & 9.57s          & 332.11s        & 963.44s        & 235.38s        \\
                     & BEKM    & 105.58s        & 20.42s         & 18.14s         & 100.69s        & 198.04s        & 1610.27s       \\
                     & BLPA    & 101.28s        & 21.49s         & 19.09s         & 108.43s        & 231.09s        & 1595.64s       \\
                     & GIF     & 1.12s          & 0.65s          & 0.59s          & 2.10s          & 5.09s          & 11.58s         \\
                     & SR      & 36.00s         & 12.73s         & 12.12s         & 57.75s         & 169.89s        & OOM            \\
                     & Fast    & 35.85s         & 12.25s         & 12.08s         & 55.81s         & 159.12s        & OOM            \\
                     & MEGU    & \underline{0.37s}          & \underline{0.24s}          & \underline{0.22s}          & \underline{0.66s}          & \underline{1.37s}          & \underline{3.67s}          \\
                     & ETR     & \textbf{0.02s} & \textbf{0.02s} & \textbf{0.02s} & \textbf{0.03s} & \textbf{0.06s} & \textbf{0.08s} \\ \hline
\end{tabular}
\caption{Performance comparison in terms of unlearning efficiency in the feature unlearning task.}
\label{tab:feature_2}
\end{table*}

\begin{table*}[!t]
\centering
\begin{tabular}{cc|cccccc}
\hline
Backbone             & Method & PubMed          & CiteSeer         & Cora             & CS               & Physics          & ogbn-arxiv      \\ \hline
\multirow{7}{*}{GCN} & BEKM   & 0.095           & 0.10             & 0.0098           & 0.43             & 0.049            & 1.48            \\
                     & BLPA   & 0.13            & 0.095            & 0.0088           & 0.43             & 0.050            & 1.53            \\
                     & GIF    & 0.072           & 0.0045           & 0.029            & 0.029            & 0.22             & 1.44            \\
                     & SR     & 0.085           & 0.051            & 0.016            & 0.0088           & 0.0019           & OOM             \\
                     & Fast   & 0.083           & 0.050            & 0.016            & 0.0088           & 0.0019           & OOM             \\
                     & MEGU   & \textbf{0.0085} & \underline{0.0023}           & \underline{0.0033}           & \underline{0.0012}           & \underline{0.00059}          & \underline{0.041}           \\
                     & ETR    & \underline{0.015}           & \textbf{0.00015} & \textbf{0.00021} & \textbf{0.00026} & \textbf{0.00011} & \textbf{0.013}  \\ \hline
\multirow{7}{*}{GAT} & BEKM   & \textbf{0.010}  & 0.015            & 0.030            & 0.015            & 0.0065           & 0.12            \\
                     & BLPA   & \underline{0.011}           & 0.015            & 0.031            & 0.015            & 0.005            & 0.12            \\
                     & GIF    & 0.10            & 0.0069           & 0.014            & \underline{0.0017}           & 0.0079           & 1.86            \\
                     & SR     & 0.036           & 0.012            & 0.011            & 0.031            & 0.086            & OOM             \\
                     & Fast   & 0.034           & 0.014            & 0.011            & 0.031            & 0.086            & OOM             \\
                     & MEGU   & 0.082           & \underline{0.0012}           & \underline{0.0023}           & 0.023            & \underline{0.0022}           & \textbf{0.0033} \\
                     & ETR    & \underline{0.011}           & \textbf{0.00040} & \textbf{0.00041} & \textbf{0.00029} & \textbf{0.00014} & \underline{0.016}           \\ \hline
\end{tabular}
\caption{Performance comparison in terms of unlearning efficacy in the edge unlearning task.}
\label{tab:edge_3}
\end{table*}

\begin{table*}[!t]
\centering
\begin{tabular}{cc|cccccc}
\hline
Backbone             & Method & PubMed          & CiteSeer         & Cora             & CS               & Physics          & ogbn-arxiv      \\ \hline
\multirow{7}{*}{GCN} & BEKM   & 0.091           & 0.018            & 0.0092           & 0.035            & 0.055            & 1.43            \\
                     & BLPA   & 0.097           & 0.014            & 0.0092           & 0.036            & 0.059            & 1.52            \\
                     & GIF    & 0.072           & 0.0045           & 0.024            & 0.0098           & 0.035            & 1.55            \\
                     & SR     & 0.082           & 0.0080           & \underline{0.0090}           & 0.0087           & 0.0023           & OOM             \\
                     & Fast   & 0.081           & 0.0079           & 0.0093           & 0.0088           & 0.0023           & OOM             \\
                     & MEGU   & \textbf{0.0084} & \underline{0.0024}           & 0.019            & \underline{0.00035}          & \underline{0.00058}          & \underline{0.041}           \\
                     & ETR    & \underline{0.0090}          & \textbf{0.00011} & \textbf{0.00011} & \textbf{0.00023} & \textbf{0.00010} & \textbf{0.013}  \\ \hline
\multirow{7}{*}{GAT} & BEKM   & \underline{0.011}           & 0.014            & 0.051            & 0.0051           & 0.0084           & 0.65            \\
                     & BLPA   & \underline{0.011}           & 0.013            & 0.055            & 0.0051           & \underline{0.0069}           & 0.66            \\
                     & GIF    & 0.10            & 0.0070           & 0.040            & 0.033            & 0.0089           & 1.79            \\
                     & SR     & 0.024           & 0.0036           & 0.0073           & \underline{0.0042}           & 0.097            & OOM             \\
                     & Fast   & 0.025           & 0.0036           & 0.0072           & 0.0043           & 0.097            & OOM             \\
                     & MEGU   & 0.081           & \underline{0.0012}           & \underline{0.0023}           & 0.023            & 0.012            & \textbf{0.0032} \\
                     & ETR    & \textbf{0.0061} & \textbf{0.00029} & \textbf{0.00023} & \textbf{0.00028} & \textbf{0.00014} & \underline{0.012}           \\ \hline
\end{tabular}
\caption{Performance comparison in terms of unlearning efficacy in the feature unlearning task.}
\label{tab:feature_3}
\end{table*}

\end{document}